\documentclass[10pt]{article}
\usepackage{graphicx}
\usepackage{epsfig}
\usepackage{amsmath,amsthm, amssymb}
\usepackage{url}
\usepackage[a4paper, lmargin=1.1in, rmargin=1.0in, tmargin=1.2in, bmargin=1.0in]{geometry}

\newcommand{\lan}{{\langle}}
\newcommand{\ran}{{\rangle}}

\newtheorem{ob}{Observer Abstraction}
\newtheorem{ax}{Axiom}
\newtheorem{df}{Definition}
\newtheorem{lemma}{Lemma}
\newtheorem{claim}{Claim}
\begin{document}


\title{Towards a Framework for Observing Artificial Evolutionary Systems}
\author{Janardan Misra \\ HTS Research, Bangalore, India 560076 \\ {\sf Email:
janardan.misra@honeywell.com}}
\date{}
\maketitle
\tableofcontents

\begin{abstract}
Establishing the emergence of evolutionary behavior as a defining
characteristic of `life' is a major step in the Artificial life
(ALife) studies. We present here an abstract formal framework for
this aim based upon the notion of high-level observations made on
the ALife model at hand during its simulations. An observation
process is defined as a computable transformation from the
underlying dynamic structure of the model universe to a tuple
consisting of abstract components needed to establish the
evolutionary processes in the model. Starting with defining
entities and their evolutionary relationships observed during the
simulations of the model, the framework prescribes a series of
definitions, followed by the axioms (conditions) that must be met
in order to establish the level of evolutionary behavior in the
model. The examples of Cellular Automata based Langton Loops and
$\lambda$ calculus based Algorithmic Chemistry are used to
illustrate the framework. Generic design suggestions for the ALife
research are also drawn based upon the framework design and case
study analysis.
\end{abstract}

\newpage
\section{Background}

The phenomenon of ``life" on earth is one of the most intriguing
one with vast variety and complexity of forms in which it is found
on multiple levels ranging from microbiological scale to higher
taxa exhibiting a wide array of characteristics. Although we can
explain several aspects of life around us in the light of existing
theories for real-life evolution, we do not yet have a
comprehensive understanding of the principles underlying the
emergence of life and the conditions that led to the diversity and
complexity of life on earth \cite{Fut98}. Experimental methods to
understand biological processes are usually difficult and error
prone because living systems are by nature complex in design and
usually hard to manipulate. Evolution is even more difficult to
study experimentally since experiments may span over several
generations and are usually difficult to control.

\emph{Artificial life} (ALife) is an elegant methodology to
complement real life theories to study the principles underlying
the complex phenomena of life without directly working with the
real-life organisms.
For example, ALife studies can complement theoretical biology by
uncovering detailed dynamics of evolution where real life
experiments are not possible, and by developing generalized formal
models for life to determine criterions so that life in any
arbitrary model can be observed.

Cellular Automata based models are one of the earliest attempts of
synthesis to understand the underlying logic of self reproduction
\cite{Sipper98}. Later attempts in the field have considered
several new kinds of synthetic structures including programs,
$\lambda$ terms, strings, graphs, automata's etc (see for
overview~\cite{ac:Dittrich01}) and have demonstrated that one or
the other observable properties of real-life are shared by all of
these models, though the parallel diversity and robust evolving
structures which we find in real-life are yet to be designed. One
of the guiding principles of ALife research behind these novel
class of synthetic structures is that - ``life is a property of
form and organization rather than the matter used to build it"
\cite{Langton95}. This criterion to identify life in these novel
synthetic structures in turn poses further questions as to which
kind of organizational structures possess life? Which properties
should we be looking at in those structures? and most importantly
how can we recognize life in any arbitrary model?

To partly address these questions, in this paper, \emph{we proceed
with the hypothesis that one of the possible ways life can be
recognized in an arbitrary ALife model during its simulations is
by observing population of entities undergoing evolution in the
sprit of evolution by natural selection, which demands the
presence of reproduction, heredity, variation owing to mutational
changes, and finally natural selection based reproductive success.
(See also~\cite{Dawkins82}).} Though the criterion to equate life
with the presence of evolutionary processes excludes other
plausible properties including metabolism~\cite{ac:BFF92},
complexity~\cite{ac:adami00}, self organization~\cite{ac:Kau93},
autonomy and autopoiesis~\cite{Zeleny81}, yet captures a wide
class of interesting phenomena related to population level
evolution of entities~\cite{SS00}. Such a identification of
population level evolutionary phenomena in arbitrary ALife models
critically depends upon the observations carried out over
simulations as we discuss next.

\subsection{Motivations}

Observations play a fundamental role both in real life studies as
well as in ALife research. In case of real-life studies the role
of observations is usually limited to an experimental analysis to
uncover the specific dynamics underlying the observed life forms
and their properties using natural observations or controlled
experiments. On the other hand, in case of ALife studies, in
general there is no known method to decide beforehand the kind of
entities, which might demonstrate non-trivial life-like behavior,
without closely observing the simulations of the model.

The very identification of life is thus an existential problem for
ALife studies and we need some sound formal framework to address
this problem. In absence of a formal framework, we often encounter
intuitive and informal arguments, which remain useful only to
specific models and do not always have the generic perspective. We
question whether these model-specific arguments are sufficient to
support the presence of an extremely complex phenomenon such as
evolution in ALife models. Without formal foundations to ascertain
these (informally presented) claims, there is always a danger to
run into conflicting arguments, which might, for example, be based
upon observations of the smulations on different levels. Nehaniv
and Dautenhahn~\cite{ND98} specifically discuss that
identification of time varying entities is a deep rooted problem
in the context of formal definitions for self-reproduction and add
that in absence of observers it is problematic to decide whether
an instance of artificial self-replication be treated at all a
life-like one.

In an attempt to provide a formal platform for observations in
ALife studies, a high level abstraction mechanism is presented for
characterizing the observations needed to establish the
evolutionary behavior in ALife studies. Initial concepts in this
direction appeared in~\cite{Misra06a,Misra07}. The central concept
of the framework is the formalization of the observation process,
which we believe is essential, but most often remains implicit in
ALife studies. The observation process leads to abstractions on
the model universe, which are consequently used for establishing
the necessary elements and the level of evolution in that model.
Examples of Cellular Automata based Langton Loops
(Section~\ref{chap:langton}) and $\lambda$ calculus based
Algorithmic Chemistry (Section~\ref{chap:lambda}) are used to
demonstrate the applicability of the formalism. Importantly the
framework does not build upon the low-level dynamics or the
``physical laws" of the underlying universe of the particular
ALife model at hand, and thus permits the study of higher-level
observationally ``emergent" phenomena as a basis of evolution.

\subsection{Contributions}

The paper brings the implicitly assumed notion of
\emph{observations to be carried out independent of the underlying
structure of the model} into main focus of ALife studies. It was
not clear before that observational processes can be independently
studied in their own right and the work presented in this paper
makes it clear by placing observations into distinct formal
platform. The work can also be seen as an attempt to fulfill the
need for explicitly separating the design of the ALife models from
the abstractions used to describe their dynamic progression.


The approach has helped us to formalize certain aspects of life
including recognition of reproductive relationships under parental
mutations as well as reproductive mutations in children along with
their epigenetic developments, which were believed to be difficult
to formalize before~\cite{ND98,Nehaniv05}. The formalism captures
wide range of reproductive instances including the case of multi
parent reproduction (without resorting to the concept of species),
and the case of reproduction without overall growth of the
population (cf.~\cite{ND98}). Finally framework design and
analysis of the case studies are used to draw useful design
suggestions for the ALife research so that interesting
evolutionary phenomena involving life-like entities can be better
synthesized and analyzed.\\


The paper is organized as follows: In
Section~\ref{chap:framework}, we will formally elaborate the
framework. Case studies will follow in
Section~\ref{sec:case-studies} -- Section~\ref{chap:langton}
applies the framework on cellular automata based Langton's Loops
and Section~\ref{chap:lambda} on $\lambda$ calculus based
Algorithmic chemistry. Section~\ref{chap:related} presents a
discussion of related work, and is followed by concluding remarks
in Section~\ref{sec:concluding}, along with the discussion on
design suggestions for the ALife researchers in
Section~\ref{sec:design-suggst}. Limitations of the framework and
pointers for further work are discussed in
Sections~\ref{sec:limit} and~\ref{sec:further-work} respectively.

\section{The Framework} \label{chap:framework}

In the ensuing discussion, we will use ``ALife model" and
``model", ``Observation process" and ``Observer" interchangeably
to add convenience in presentation. Similarly ``real-life" is used
in the paper to refer to organic life on earth in contrast with
the ``artificial-life". Also, \emph{Observer Abstractions} will
refer to specific observations and corresponding abstractions made
upon the ALife model during its simulations. \emph{Axioms} are
used to specify conditions which need to be satisfied in order to
infer various components of evolution. Thus for each fundamental
component of evolution: self reproduction, mutation, heredity, and
natural selection, framework specifies certain Axioms constraining
what is needed to be observed and consequently inferred in a
formal way if any claim towards presence of any of these
evolutionary components has to be substantiated. The aim is to
define these formal Axioms such that only valid claims for
evolutionary processes in a model can be entertained.
%
Auxiliary formal structures are used in the intermediate stages of
analysis. E.g., distance measure for determining dissimilarity
between entities for their specific characteristics (see
Section~\ref{sec:distance}).

\subsection{The Formal Structure of the Framework}
\label{chap:formalstruct}

To illustrate the framework, we will use a simple example of a
binary string based chemistry whenever required in the discussion
to assist the intuition behind the formalism. The chemistry will
be referred as $\mathbf{CBS}$ (Chemistry of Binary Strings).
Specifics of the design and structure of the chemistry will be
explained as we proceed.

\subsubsection{Observation Process and the Model Universe}
\label{sec:obproc}

We define the observation process as a transformation from the
underlying universe of the ALife model to a set of observed
abstractions as follows:\\\\
\textbf{Observation Process.} { \em $\Gamma \mapsto_{Obj} \Pi$: An
observation process $Obj$ is defined as a computable
transformation from the underlying model structure $\Gamma =
(\Sigma, \mathcal{T})$ to observer abstractions $\Pi = (E$, $F$,
$\Upsilon$, $D$, $\delta_{mut}$, $\delta_{rep\_mut}$, $C)$ and
represented as $\Gamma \mapsto_{Obj} \Pi$.
$\Gamma$, and $\Pi$ are defined below.}\\

The condition of computability is to ensure that the framework is
decidable (or feasible~\cite{Bedau99}), that is, the observation
process only involves feasible computable steps, which can also be
algorithmically programmed by the designer of the model and that
infeasible observations defined in terms of non verifiable claims
(e.g., `meta - information' based claims) can be avoided.\\\\
\textbf{States}. {\em $\Sigma$: set of observed states of the
model in a simulation.}\\

The exact definition of a ``state" would vary from one model to
the other due to their irreducible design differences as well as
the level at which observations are being made. A
multiset~\footnote{A \emph{multiset} $M$ on a set $E$ is a mapping
associating nonnegative integers (representing multiplicities)
with each element of $E$, $M: E \rightarrow \mathcal{N}$.
Informally a multiset may contain multiple copies of its
elements.} can sometimes be used to represent state of a model by
defining it as a collection of observable basic structures and
their corresponding multiplicities in the model at any instance
during its simulation. As an example, we can consider an observed
state in the case of our example chemistry of binary strings,
$\mathbf{CBS}$, as a multiset - such that some specific state
could be -
\[\{ (00101, 2), (10101, 1), (010, 1), (0100, 1) (10100, 1)\}\]

 Further illustrative examples can be seen in the case studies appearing in
Section~\ref{sec:case-studies}. \\\\
\textbf{Observed Run}. {\em $\mathcal{T}$: set of observed
sequences of states, ordered with respect to the temporal
progression of the model. Each such sequence represents one
\emph{observed run} of the model. A sequence of states is formally
represented as a mapping: $N \rightarrow \Sigma$, where $N$ is the
set of non negative integers acting as a set of indexes for the
states in the sequence.}\\

A temporally ordered state sequence is one of the basic building
blocks in the framework upon which all other observed abstractions
are made. Such a definition of a run of model implicitly implies
that the framework is fundamentally based upon the dynamic
simulations of the model and not upon static analytical
inferences. This is in accord with the notion of ``weak
emergence"~\cite{ac:BMPRAGIKR00}, which is a generic
characteristic of most of the ALife studies.

For a state $S$, $S-1$, and $S+1$ would denote the the states just
before and after $S$ in a state sequence.\\

$\Sigma\ \textit{and}\ \mathcal{T}$ thus define the underlying
dynamic structure of the model $\Gamma = (\Sigma, \mathcal{T})$.
Using $\Sigma\ \textit{and}\ \mathcal{T}$, sometimes, a state
machine model can also be used to define $\Gamma$.

\subsubsection{Entities and Their Characteristics}
\label{sec:observe}

\begin{ob}[\textbf{Entity Set}]
$E$: set of entities observed and uniquely identified by the
observer within a state and across the states of the model.
\end{ob}

The criterion to select the set of uniquely identifiable entities
in a given state of the ALife model is entirely dependent on the
observation process as specified by the ALife researcher. Thus for
the same set of simulations of a model, there may exist very
different observed states as well as entities. Nonetheless, same
observation process must not yield different sets of entities in
two identical states.

Defining sound criterion to identify entities often requires a
careful a attention since arbitrariness in defining entities might
well lead to the problem of false positives as discussed later
(see Section~\ref{false+ve}.)

``Tagging" can be sometimes used as a mechanism for the
identification of individual entities whenever there exist
multiple entities in the same state which are otherwise
indistinguishable. Thus an observer may associate and
correspondingly identify every entity in a state using a unique
tag. In cases, where tags are selected such that they remain
invariant under time progression of the model (i.e., do not change
owing to reactions or interactions of the entities), the tags can
as well be used for recognition of the persistence of these
entities across the states of the model.

For example, in case of $\mathbf{CBS}$, an observer might identify
individual strings as entities such that to distinguish
syntactically identical strings, we can associate with every
string an integer tag such that with tag $i$, an entity
corresponding to the binary string $s$ can be represented as
$[s]_i$. Thus a possible set of entities corresponding to the
example state given above becomes
$$\{ [00101]_1, [00101]_2, [10101]_3, [010]_4, [0100]_5,
[10100]_6\}$$ Alternately another observer may choose to define
entities as a tuples consisting of strings with \emph{three
identical leftmost bits} - giving the set of entities for the same
state as $$\{ [\underline{001}01, \underline{001}01]_1,
[\underline{101}01, \underline{101}00]_2, [\underline{010},
\underline{010}0]_3\}$$

\begin{ob}[\textbf{State Function}]
$F \subseteq E\ \times \Sigma$ returns the state(s) in the state
sequences in which a particular entity is observed. For a specific
state sequence $F$ can be treated as a function.
\end{ob}

The state information provided by $F$ for entities will be used
later to define valid evolutionary relationships among them. In
general observers may use different mechanisms based upon the
nature of model as well as the entities defined, to determine the
state for a given entity. For example, as a simple mechanism, in
case of $\mathbf{CBS}$, the observer can maintain a table mapping
entities to their corresponding states in order to define
$F$.\\

Having defined the sequence of states with temporal ordering and
the entities identified by their tags, we will now proceed to
discuss how an observer might define the detailed observable
characteristics for such entities. Using these characteristics it
can draw descendent relationship, as well as can establish
presence of other components of evolution, e.g., heredity and
variation. To this aim, we will define `character space', as set
of values for the observed characteristics. These values might be
purely symbolic without any relative ordering or can be ordered
using suitable ordering relation.\\

\begin{ob}[\textbf{Character Space}]
The observer should define the set of all possible orthogonal and
{\sf measurable} characteristics for possible entities in the
model as a multi dimensional character space $\Upsilon =
\mathit{Char}_1 \times \mathit{Char}_2 \times \ldots \times
\mathit{Char}_n$, where each of $\mathit{Char}_i$ is the set of
values for $i^{th}$ characteristic. Each of $\mathit{Char}_i$ make
one dimension in the space $\Upsilon$. Each entity $e \in E$ is
thus a point in $\Upsilon$, say $e\ =\ (v_1, v_2, \ldots v_n)$,
where $v_i \in \mathit{Char}_i$.
\end{ob}
For a vector $x = (a_1, a_2, \ldots, a_r), i^{th}$ element ($a_i$)
will be denoted as $x[i]$. For some of the characteristics
observer might define a `partial ordering' ($\leq_i$ for $Char_i
\in \Upsilon$), which can be used to compare values for those
characteristics. The absence of any characteristics in an entity
is represented by special zero element $0_{char_i}$ such that if
$\mathit{Char}_i$ is(partially) ordered then $\forall v \in
\mathit{Char}_i$. $0_{char_i} \leq_i v$.

Notice that, observable characteristics need not to be limited to
syntactic level  or \emph{structural properties} and can also
include semantic properties - \emph{observable patterns of
behaviors}. Though semantic properties are much more difficult to
observe and measure than the syntactic ones since they require
abstracting the patterns of reactions over a range of states.

In case of $\mathbf{CBS}$, for simplicity we may assume that model
consists of binary strings of size $n$. In that case each position
of the string can represent one orthogonal dimension and we have
only two binary values (\{0, 1\}) at any position in a string for
corresponding dimension. Thus character space $\Upsilon$ in
$\mathbf{CBS}$ is $n$ dimensional binary hypercube with each
string occupying a possible diagonal end point. We will represent
this hypercube as $\{0,\ 1\}^n$. The ordering relation $\leq$ for
all dimensions is the same and defined as $0 < 1$.

In terms of such character space $\Upsilon$, an entity set $E$ at
any state can be defined by annotating the points in $\Upsilon$
with integer constants denoting the multiplicity of the entities
present in $E$ with characteristics defined by the point.

\subsubsection{Distance Measures}
\label{sec:distance}

Another important structure in the framework is the
``dissimilarity measure" ($D$) to define the ``observable
differences" ($\mathbf{Diff}$) between the characteristics of the
entities in a population. The distance measure defined below can
be used by the observer to distribute entities into separate
clusters such that entities in the same cluster are sufficiently
similar while entities from different clusters are distinguishably
different in their characteristics. Again exact definition of
distance function is
model dependent.\\

\begin{ob}[\textbf{Distance Measure}] An observer defines a decidable
clustering distance measure $D: E \times E \rightarrow
\mathbf{Diff}$, where $\mathbf{Diff}$ is the set of values to
characterize the observable ``differences" between entities in
$E$.
\end{ob}


Examples include the Hamming distance to define distance between
genomic strings in the Eigen's model of molecular evolution
\cite{ps01}, set of points where two computable functions differ
in their function graphs, or the set of instructions where two
programs may differ. One of the known criterions to define the
concept of species is ``phenotype similarity" \cite{Ridley96},
which can also be seen as another example for distance measure.



In case of $\mathbf{CBS}$, we can define an auxiliary function
$\oplus: \{0, 1\} \times \{0, 1\} \rightarrow  \{0, 1\}$ as a
binary $XOR$ such that we have $0\oplus0 = 1\oplus1 = 0$, and
$1\oplus 0 = 0\oplus1 = 1$. Thus the clustering distance measure
$D: E \times E \rightarrow \{0, 1\}^n$ is defined such that
$\forall i. D(e_1, e_2)[i] = e_1[i] \oplus e_2[i]$, which implies
that $\mathbf{Diff}$ = $\{0, 1 \}^n$. For example in case of two
$n = 3$ bit binary entities $e_1 = [001]_1$ and $e_2 = [101]_2$,
$D(e_1, e_2) = 100$. Other alternatives may include Hamming
distance measure $D(e_1, e_2) = \sum_{i=1}^n (e_1[i] \oplus
e_2[i])$ with $\mathbf{Diff}$ = $\{0, 1, \ldots, n\}$.

\subsubsection{Observable Limits on Mutational Changes}
\label{sec:limits}

The observer needs to specify the limits under which it can
recognize an entity across states even in the presence of
mutational changes in the entity owing to its interactions with
the environment. This is an inherent limiting property on the part
of the observer and could vary among observers. Based upon the
limit referred here as $\delta_{mut}$, an observer can establish
whether two entities in different successive states are indeed the
same with differences owning to mutations or not. The smaller the
limit, the harder it will be for an observer to keep recognizing
entities across states and he would be counting mutated entities
as the new entities. As entities are observed in more and more
refined levels of details, their apparent similarities melt away
and differences become sharply noticeable.

Another type of mutations arise during reproduction, in which case
an observer has to identify whether an entity is indeed an
descendent of another entity even though they might not be
similar. This necessitates us to introduce another bound on
observable reproductive mutations as $\delta_{rep\_mut}$. This
limit on observable reproductive mutations is indeed crucial while
working with models where epigenetic development in the entities
can be observed \cite{Mah97}. This is because in such chemistries
including examples from real life, the ``child" entity and the
``parent" entities do not resemble with each other at the
beginning and observer has to wait until whole epigenetic
developmental process gets unfolded and then compare the entities
for similarities in their characteristics.  $\delta_{rep\_mut}$
assists an observer to establish whether a particular entity could
be treated as a ``descendent" of another entity or not.

Another reason for introducing the limit $\delta_{rep\_mut}$ is
that from the view point of an high level observation process not
recording every micro level details, it is quite essential to
distinguish between parent entities and other secondary entities
involved in the reproductive process. Consider, for example, a
model where entity $A$ reproduces according to reaction $A + B
\rightarrow 2A' +C$, where $A'$ is mutant child entity of $A$,
which can be determined by an observational process only when it
can establish that $A$ and $A'$ are sufficiently similar with
respect to their characteristics, while $A'$ and $B$ are not.
These limits on
observable differences are formally defined as follows: \\

\begin{ob}[\textbf{Mutation Bounds}] Based upon the choice of clustering
distance measure $D$, the observer selects some suitable
$\delta_{mut},\ \delta_{rep\_mut} \in \mathbf{Diff}$, which will
be used later to bound mutational changes (both reproductive and
otherwise) for proper recognition. $\delta_{mut}$ and $
\delta_{rep\_mut}$ are vectors such that each element specifies an
observer-defined threshold on the recognizable mutational changes
for corresponding characteristics.
\end{ob}

It is important to note that the choice of $\delta_{mut},\
\delta_{rep\_mut}$ critically affects further inferences. For
example, a choice of very large values would result in the lack of
identification of variability in characteristics and thus make it
difficult to infer natural selection (discussed later). On the
other hand if an observer decides to select very small values for
$\delta_{mut}$ then it cannot recognize persistence of an entity
across states under changes, similarly small values for
$\delta_{rep\_mut}$ make it harder to establish reproductive
relationship among entities and for such an observer every new
entity would seem to be appearing \emph{de novo} in the model.


\subsection{Evolutionary Components}
\label{chap:evocomponents}

Having defined the observation process as a computable
transformation from the underlying sequence of observed states of
the model to the set of components involving entities and their
observable characteristics with measurable differences as well as
observable limits on such differences, we will now proceed with
formalization of the fundamental evolutionary components:
mutations, reproduction, heredity and natural selection.

\subsubsection{Mutations}
\label{mutformalized}

For evolution to be effective entities should change (mutate) over
the course of their interaction with the environment (or other
entities.) Moreover, there can also be observable differences
between the child and the parent entities arising out of
reproductive processes. These changes in the characteristics of
the entities may or may not be inheritable based upon the design
of the model and the simulation instance.

Mutations can be considered of carrying two kinds of effects in
the entities: one where mutations change the values for specific
characteristics, secondly where after mutation an entity has at
least one new character not present before or when certain
characteristics are lost. We define a \emph{Recognition relation}
to establish the non reproductive
mutational changes in the entities: \\

\begin{df}[\textbf{Recognition Relation}]
The observer establishes recognition of entities across states of
the model with (or without) mutations by defining the function
$\mathbf{R_{\delta_{mut}}}$: $E \rightsquigarrow E$, which is a
partial function and satisfies the following axioms: \\
\end{df}

\begin{ax}
$\forall e, e' \in E\ .\  \mathbf{R_{\delta_{mut}}}$$(e) = e'
\Rightarrow F(e') = F(e) + 1$.
\end{ax}

Informally, the axiom states that entities to be recognized as the
same even with mutational changes have be observed in successive
states. $\mathbf{R_{\delta_{mut}}}$ is defined anti symmetric to
ensure that entities are recognized based upon the time
progression of the model not in any other arbitrary order.\\

\begin{ax} $\mathbf{R_{\delta_{mut}}}$ is an injective function,
that is, $\forall e, e' \in E$. $\mathbf{R_{\delta_{mut}}}$$(e) =$
$\mathbf{R_{\delta_{mut}}}$$(e')$ $\Rightarrow e = e'$
\end{ax}

Informally, the axiom states that no two different entities in one
state can be recognized as the same in the next state. \\

\begin{ax}
$\forall e, e' \in E$. $\forall \mathit{Char}_i \in \Upsilon$.
$\mathbf{R_{\delta_{mut}}}$$(e) = e' \Rightarrow
0_{\mathit{diff}_i} \preceq_i D(e, e')[i] \preceq_i
\delta_{mut}[i]$
\end{ax}

Informally $\mathbf{R_{\delta_{mut}}}$$(e)$ is that $e' \in E$,
which is recognized in the next state by the observer as $e$ in
the previous state with possible mutations bounded by
$\delta_{mut}$. In other words if entity $e$ mutates and changes
in the next state and identified as $e'$, then observer might be
able to recognize $e$ and $e'$ as the same if these changes
(between $e$ and $e'$) are bounded by $\delta_{mut}$.

\subsubsection{Reproduction}
\label{reproformalized}

Reproduction is one of the fundamental components of evolution.
Through reproduction, entities pass on their characteristics to
the next generation and increase the population size. Reproduction
is possibly the only way by which abstract entity structures can
persist across generations in case of those Alife models, where
entities do not persist forever. In our framework, the way an
observer establishes reproduction is by providing observed
evidence for it. This is done by defining causal descendence
relationships among the entities across states. The parent and the
child entities are recognized by the observer as being
sufficiently similar and ``causally'' connected across the
states:\\

\begin{ob}[\textbf{Observed Causality}]
$C \subseteq E \times E\ .\  C$ establishes the observed causality
among the entities appearing in the successive states. $C$
satisfies the following axiom: \\
\end{ob}

\begin{ax}[\textbf{Causality}]
$\forall e, e' \in E\ .\  (e, e') \in C \Rightarrow [F(e') = F(e)
+ 1] \wedge [\not\!\exists e'' \in E\ .\  F(e'') = F(e) \wedge
\mathbf{R_{\delta_{mut}}}$ $(e'') = e']$
\end{ax}

Informally, the axiom on causal relationship $C$ states that, if
an entity $e$ is causally connected to another entity $e'$, then
the observer must observe $e'$ in the next state of $e$ and never
before. This is to ensure that mutations are not confused by the
observer with reproductions. Notice that in order to establish
causal relation between entities, observers need not necessarily
know the underlying reaction semantics or the micro level dynamics
of the model. Only requirement is that the observer's claimed
causality conforms with the stated axiom. In essence, this
formulation of causality is an abstract specification which
demands observers to identify the entities which have been
observed to be causal sources for the appearance of a new entity.
Only then proper descendance relation for the new entity can be
established.

Apart from causality $C$ we also need auxiliary relation $\Delta$
to determine that the differences due to the reproductive
mutations are also bounded by $\delta_{rep\_mut}$. \\

\begin{df}
$\Delta \subseteq E \times E$ such that $\forall e, e' \in E \ .\
(e, e') \in \Delta \Leftrightarrow \forall \mathit{Char}_i \in
\Upsilon\ .$  if $\mathit{Char}_i$ has an ordering then $D(e,
e')[i] \preceq_i \delta_{rep\_mut}[i]$.
\end{df}

Informally for $(e, e'$) to be in $\Delta$, their differences for
each single characteristic $Char_i$ must be bounded by
$\delta_{rep\_mut}[i]$. \\

Based on the thus established notion of ``causal'' relationships
between entities and $\Delta$, we will define
$\mathbf{AncestorOf}$ relation, which connects entities for which
an observer can establish descendence relationship across
generations. \\

\begin{df}$\mathbf{AncestorOf}$ = $(\ (C \ \cup \
\mathbf{R_{\delta_{mut}}})^+ \ \cap \ \Delta)^+$
\end{df}

In this definition the (inner) transitive closure of $(C \ \cup \
\mathbf{R_{\delta_{mut}}})$ captures the observed causality ($C$)
across multiple states even in cases when ``parent" entities might
undergo mutational changes ($\mathbf{R_{\delta_{mut}}}$) before
``child" entities complete their ``epigenetic" maturation with
possible reproductive mutations. Intersection with $\Delta$
ensures that causally related parent and child entities are not
too different from each other, that is, reproductive mutational
changes are under observable limit. Outer transitive closure is to
make $\mathbf{AncestorOf}$ relationship transitive in nature so
that entities in the same lineage can be related with each other.
For $e, e' \in E$, $(e, e') \in$ $\mathbf{AncestorOf}$, describes
that $e$ is observed as an ancestor of $e'$.


\begin{figure}
\centering
\includegraphics[scale=0.7]{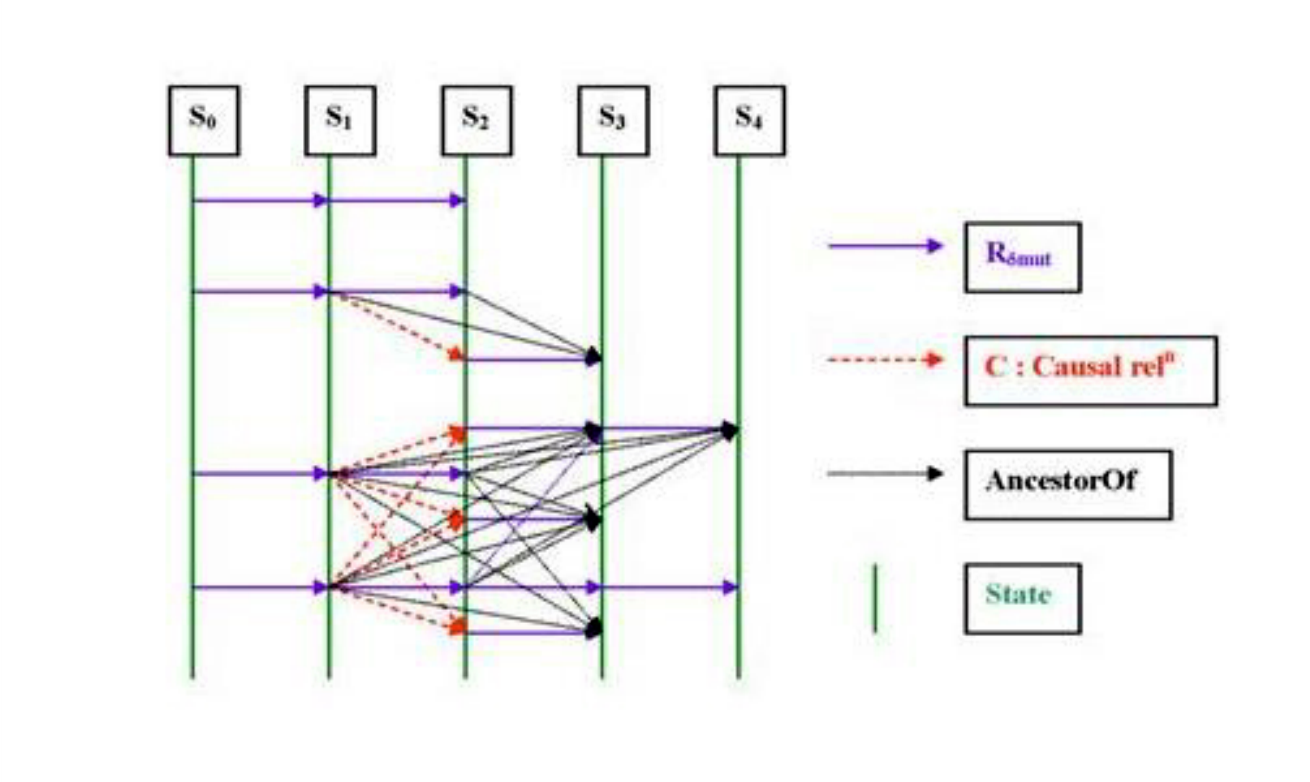}
\caption[Graphical view of evolutionary relations]{Graphical view
of the relationships between entities in successive states.
Recognition relation $\mathbf{Rec}$, Causal relation $C$, and
$\mathbf{AncestorOf}$.} \label{fig:hassediagram}
\end{figure}

Figure~\ref{fig:hassediagram} depicts graphically the
relationships between entities in successive states. Vertical
lines represent the states ($S_0, S_1, S_2, S_3, S_4$). Various
kinds of arrows represent different relationships: recognition
relation $\mathbf{R_{\delta_{mut}}}$, causal relation $C$, and
$\mathbf{AncestorOf}$. The end points of the arrows on state lines
represent entities.

\begin{claim}
Case of Reflexive Autocatalysis.
\end{claim}
\begin{proof}
In the simplest form, a reflexive autocatalytic cycle is
represented as a system of reaction equations:
\begin{eqnarray*}
  A + X_1 &=& A_1 + Y_1 \\
  A_1 + X_2 &=& A_2 + Y_2 \\
  \vdots \\
  A_{n-1} + X_{n} &=& mA' + Y_n
\end{eqnarray*}
where $m$ copies of entity $A'$ are produced at the end and that
entity $A'$ is a variation of entity $A$, i.e., $(A, A') \in
\Delta$. Such autocataltic cycles are supposed to be the chemical
basis of biological growth and reproduction. Examples include the
Calvin cycle, reductive citric acid cycle, and the formose system.
Competing cycles of this sort can even undergo limited evolution,
though they are supposed to have very limited
heredity~\cite{ss97}.

In the current framework suppose an observer could determine the
causal relations - $(A, A_1)$, $(A_1, A_2)$, $\ldots$, $(A_{n-1},
A')$. Also assume that entity $A$ does not undergo any changes
before $A'$ is produced, that is, $(A, A) \in R_{mut}$. Then $(C \
\cup \ \mathbf{R_{\delta_{mut}}})^+$ would contain $(A, A')$ so
also would $(\ (C \ \cup \ \mathbf{R_{\delta_{mut}}})^+ \ \cap \
\Delta)$ establishing the reproduction of $A$ through reflexive
autocatalytic cycle and with variation.
 \end{proof}

\begin{claim}
Recognition of reproductive relationships under parental mutations
together with reproductive mutations and epigenetic developments
in the child entities.
\end{claim}
\begin{proof}
Let us see what it requires for establishing reproductive
relationship when (parent) entities might be undergoing changes
across states and child entities not only differ from the parent
entities owing to reproductive mutational changes but also that
there exist epigenetic developments in the child entities, which
make it harder for any observer to establish similarities between
child and parent entities by observing the child entities only in
the beginning (i.e., in the state when child entities were
observed for the first time.) Naturally it would require that an
observer observes child entities so long that their epigenetic
development unfolds completely - since in general there cannot be
any fixed limit on the number of states required for such
epigenetic development, we capture this requirement of
observations across states using transitive closure - $(C \ \cup \
\mathbf{R_{\delta_{mut}}})^+$, where $\mathbf{R_{\delta_{mut}}}$
ensures that (mutational) changes in the parent entities and also
the changes in the child entities during epigenetic development
are accounted for.

Lets us assume that in a state $S_i$, a child entity $c$ was
observed for the first time and (parent) entity $p$ present in the
state $S_{i-1}$ was observed to be casually connected to it.
Suppose that for entity $e$ its epigenetic development unfolds
through states $S_{i+1}, S_{i+2}, \ldots, S_{i+r}$ such that with
changes owing to the development $c$ was observed as $c_{1},
c_{2}, \ldots, c_{r}$ in these states with $(c, c_1), (c_1, c_2),
\ldots, (c_{r-1}, c_r) \in \mathbf{R_{\delta_{mut}}}$. Similarly
suppose that parent entity $p$ undergoes mutations in these
successive states and observed as $p_{1}, p_{2}. \ldots, p_{r}$
such that $(p, p_1), (p_1, p_2), \ldots, (p_{r-1}, p_r) \in
\mathbf{R_{\delta_{mut}}}$. It is clear that $(C \ \cup \
\mathbf{R_{\delta_{mut}}})^+$ would contain $(p, c), (p, c_1),
\ldots, (p, c_r)$, $\ldots, (p_r, c), (p_r, c_1), \ldots, (p_r,
c_r)$ among other tuples implying that the intersection of $(C \
\cup \ \mathbf{R_{\delta_{mut}}})^+$ with $\Delta$ would result in
those tuples $(p_m, c_n)$, where $p_m$ and $c_n$ are sufficiently
similar in their characteristic. Therefore if the resultant set
$(\ (C \ \cup \ \mathbf{R_{\delta_{mut}}})^+ \ \cap \ \Delta)^+$
is not empty, the observer can establish the reproductive
relationship between entities $p$ and $c$ even under parental
mutational changes and the epigenetic changes and reproductive
mutations in the child entity.
\end{proof}

Using $\mathbf{AncestorOf}$ relation, we now can consider the
cases of \emph{entity level reproduction} and \emph{Fecundity}:

\subsubsection*{Case 1: Entity Level Reproduction}

We consider the case where instances of individual entities can be
observed as reproducing even though there might not be any
observable increase in the size of the whole population.

For a given simulation of the model, an observer defines
the following $\mathbf{Parent_{\Delta}}$ relation: \\
\begin{df}
\begin{equation*}
\begin{split}
\mathbf{Parent_{\Delta}} = \{ & (p, c) \in \mathbf{AncestorOf} \mid\\
& \not\!\exists e \in E\ .\  [(p, e) \in \mathbf{AncestorOf}
\wedge (e, c) \in \mathbf{AncestorOf}]\}
\end{split}
\end{equation*}
\end{df}
The condition in defining $\mathbf{Parent_{\Delta}}$ is used to
ensure that $p$ is the immediate parent of $c$ and thus there is
no intermediate ancestor $e$ between $p$ and $c$.  Using
$\mathbf{Parent_{\Delta}}$ relation, in order for the observer to
establish reproduction in the model, the following axiom
should be satisfied:\\

\begin{ax}[\textbf{Reproduction}]
$\exists \textit{state sequence}\ T \in \mathcal{T} \ .\
\mathbf{Parent_{\Delta}} \neq \emptyset$
\end{ax}

This means, if there is reproduction in the model, then there
should exists some simulation $T \in \mathcal{T}$ of the model,
where at least one instance of reproduction is observed.

In case of $\mathbf{CBS}$, we consider a very simple model of
reproduction, where at any state of the model some of the strings
are randomly chosen and are copied with some random errors. How it
is done remains hidden from the observer but the observer can
observe which parent entities are chosen for copying and can
establish causal relation between these parent and their copied
child entities if the random errors occur only at even positions
as the way $\delta_{rep\_mut}$ has been defined in
Section~\ref{sec:limits}. It can be easily seen that under such
construction scheme \emph{Axiom of Reproduction} will be
satisfied.

\subsubsection*{Case 2: Population Level Reproduction - Fecundity}

Though entity level reproduction is essential to be observed, for
natural selection it is the population level collective
reproductive behavior (fecundity), which is significant owing to
the \emph{carrying capacity} of the environment. Since carrying
capacity is an limiting constraint on the maximum possible size of
population, an observer needs to establish that there is no
perpetual decline in the size of the population. In other terms
for all generations, there exists a future generation that is of
the same size or larger. This allows cyclic population sizes where
the cycle mean grows (or stays steady) over time.
Also in case of fecundity, an observer need not to observe all the
parents in the same state, nor do children need to be observed in
the same states of the model. Formally we require the observer to
establish Fecundity by satisfying the
following axiom:\\

\begin{ax}[\textbf{Fecundity}]
There exist infinitely many different generations of entities in
temporal ordering $G_1, G_2, \ldots$ such that $(\forall G_i
\subseteq E)(\exists G_{j>i} \subseteq E)\ .\ |G_j| \geq |G_i|$
where $G_j = \{c \in E \mid\  \exists a \in G_i\ .\ (a, c) \in
\mathbf{AncestorOf}\}$, (operator $|.|$ returns the size of a set.)\\
\end{ax}

Informally, the axiom states that for every generation of entities
($G_i$), in future there exist generation of its descendent
entities ($G_j$) such that the size of descendent generation must
be equal or more than current generation. Note that the
granularity of the time for determining generations is entirely
dependent on the design of the model and the observation process.

We can now formulate another important axiom from evolutionary
perspective, which asserts that reproduction in the model should
not entirely cease because of the (harmful) mutations.\\

\begin{ax}[\textbf{Preservation of Reproduction under Mutations}]
\label{continuation axiom} Some mutations do preserve
reproduction. Formally, $\exists e \in E\ .\ Ch_e = \{e' \in E:
(e, e') \in \mathbf{Parent_{\Delta}} \cup
\mathbf{R_{\delta_{mut}}}$ $ \} \neq \emptyset \Rightarrow \exists
e'' \in Ch_e\ .\  \{e' \in E: (e'', e') \in
\mathbf{Parent_{\Delta}} \}$ $ \neq \emptyset$
\end{ax}
Informally, this means, there exists entity $e \in E$, which
reproduces (with mutations) and one of those (mutant) children of
$e$ can also further reproduce. $Ch_e$ denotes the set of children
of $e$.

In case of $\mathbf{CBS}$, since copying mechanisms do not work
differently based upon selected entities, hence the errors during
copying process do preserve the above axiom of \emph{Preservation
of Reproduction under Mutations}.

\subsubsection{Heredity}
\label{heredityformalized}


Heredity, yet another precondition for evolution, can in general
be observed in two different levels: Syntactic level and Semantic
level. On \emph{syntactic level}, entity level inheritance is
implied by the structural proximity between parents and their
progenies ranging over several generations - though in case of
continuous structural changes in the parental entities and
epigenetic development in progenies, this would require an
observer to establish structural similarities over a range of
states as discussed earlier with the definition of
$\mathbf{AncestorOf}$ relation. Also for syntactic inheritance to
persist, design of the model needs to ensure that environment,
which controls the reaction semantics of entities, remains
approximately constant over a course of time so that structural
similarities also result into continued reproductive behavior.

Difficulty arises primarily on the level of multi parental
reproduction - in this situation an observer might have to
stipulate some kind of gender types and might have to relax the
mechanism of recognizing the parent-child relationship in a way as
happens for example in case of organic life, where male-female
reproductive process (often) gives birth to a progeny belonging to
``only" one gender type. In such a case, for heredity, an observer
need to ensure that, {\sf over a course of time all the gender
types are sufficiently produced in the population.}

On the other hand it is also possible to observe inheritance on
the semantic level (ignoring structural differences) in terms of
\emph{semantic relatedness} between entities, whereby an observer
can observe that progenies and their parental entities exhibit
similarities in their (reproductive) behaviors under near
identical set of environments. This in turn would require an
observer to identify the possible sequences of observable
reactions between existing entities, which appear to be yielding
new set entities (children) and in the child generation as well
there exist a similar observable reproductive process, which
enables the (re)production of entities. Such an observation would
enable the observer to abstract the reproductive processes
currently operational in the model. The inherent difficulties in
this view are obvious - in essence an observer needs to abstract
the reproductive semantics from observable reactions in the model,
which in turn might require non trivial inferences in absence of
the knowledge of the actual design of the model.

Considering the case of real-life from an observational view
point, semantic view is in fact an abstraction over all the
reproductive processes existing across various species and levels
including the case of bacterial organisms, where next generation
of bacteria may contain a mix of genetic material from various
parental bacterium of previous generation through the process of
horizontal transmission. So while in case of syntactic inheritance
an observer would only be able establish inheritance across
organisms belonging to same species, using semantic view, he could
expand his horizon to the all organic life as a whole.


However, heredity as a mechanism of preservation of syntactic
structures, appears to be crucial for those ALife models where
entities have very limited set of reproductive variations
possible, that is, where environment supports only rare forms of
entities to reproduce and any changes in the syntactic structure
of these reproductive entities may result in the elimination of
the reproductive capability. Real-life on earth as well as the
model of the Langton loops (as discussed further in
Section~\ref{chap:langton}) are definitive examples where most of
the variations in the genetic structure, or the loops
geometry/transition rules result in the loss of
reproductive/replicative
capabilities. 

Also heredity usually requires further mechanisms to reduce
possible undoing of current mutations in future generations owing
to new mutations. Therefore, in order to establish inheritance in
ALife models, sufficiently many generations of reproducing
entities need to be observed to determine that the number of
parent-child pairs where certain characteristics (both syntactic
and semantic) were inherited by child entities without further
mutations is significantly larger than those cases where mutations
altered the characteristics in the child entities. We can express
it as the following axiom:\\

\begin{ax}[\textbf{Heredity}]
Let a statistically large observed subsequence of a run $T$:
$$\Omega = lim_{N \rightarrow \infty} \langle S_n, \ldots S_N
\rangle , n \ll N$$ Consider $\mathbf{Parent_{\Delta}^{\Omega}}$ =
$\{(e, e') \in \mathbf{Parent_{\Delta}}$ $| F(e) \in \Omega \wedge
F(e') \in \Omega \}$ to be the set of all parent - child pairs
observed in $\Omega$. Again let $\mathbf{Inherited_{\Omega}^i}$ =
$\{(e, e') \in \mathbf{Parent_{\Delta}^{\Omega}}$ $| \exists
\mathit{Char}_i \in \Upsilon\ .\  D(e, e')[i] =
0_{\mathit{diff}_i} \}$ be the set of those cases of reproduction
where $i^{th}$ characteristics were inherited without (further)
mutation. Then high degree of inheritance for $i^{th}$
characteristics $\mathit{Char}_i$ implies that
$|\mathbf{Parent_{\Delta}^{\Omega}}|$/$|\mathbf{Inherited_{\Omega}^i}|$
$\simeq 1$. For syntactic inheritance to be observed in a
population of entities, we should have some such characteristics
which satisfy this condition.
\end{ax}

The axiom of heredity together with the axiom of preservation of
reproduction under mutation ensures that reproductive variation is
maintained and propagated across generations. 

\subsubsection{Natural Selection}
\label{natselformalized}

There are several existing notions of selection in the literature
on evolutionary theory
\cite{Fut98,Ridley96,Ridley97,SS00,Mah97,Kimura83}. In case of our
observation based framework we choose to define natural selection
as a \emph{statistical inference} of \emph{average reproductive
success}, which should be established by an observer on the
population of self reproducing entities over an evolutionary time
scale i.e., over statistically large number of states in a state
sequence. Other notions of selection using fitness, adaptedness,
or traits etc. are rather intricate in nature because these
concepts are relative to the specific abstraction of ``common
environment'' shared by entities and ``the environment-entity
interactions'', which are the most basic processes of selection.
Nonetheless selecting appropriate generic abstraction for these
from the point of view of an observation process is not so simple.
Therefore we consider more straightforward approach based upon the
idea that on evolutionary scale the relative reproductive success
is an effective measure, which is also an indicator of better
adaptedness or fitness. We thus define the following (necessary)
axioms for the natural selection:\\

\begin{ax}[\textbf{Observation on Evolutionary Time Scale}] An Observer must
observe statistically significant population of different
reproducing entities, say $\Lambda$ ($|\Lambda| \gg 1$), for
statistically large number of states in a state sequence $T \in
\mathcal{T}$. That is, for a statistically large subsequence
$\Omega$ of $T$, $\Omega = lim_{N \rightarrow \infty} \langle S_n,
\ldots S_N \rangle , n \ll N$, the observer defines the set of
reproducing entities $\Lambda \subseteq \bigcup_{S_j \in
\Omega}SR(S_j)$, where $SR(S_j)$ = $\{e \in E |$ $[F(e) = S_j]
\wedge [\exists e' \in E\ .\  (e, e') \in
\mathbf{Parent_{\Delta}}]\}$ is the set of all reproducing
entities in state $S_j \in \Omega$. \\
 \end{ax}

\begin{ax}[\textbf{Sorting}] Entities in $\Lambda$ should be different
with respect to characteristics in $\Upsilon$ and there should
exist differential rate of reproduction among these reproducing
entities. Rate of reproduction for an entity is the number of
child entities it reproduces before undergoing any mutations
beyond observable limit. \\
In other words, $Rate_{rep}: E \rightarrow N^+$ defined as
$\forall e \in E\ .\  Rate_{rep}(e) = |Child_e|$ where $Child_e =
\{e' \in E | \exists e'' \in E\ .\  (e'', e') \in
\mathbf{Parent_{\Delta}}$ and $[\mathbf{R_{\delta_{mut}}^+}$$(e) =
e'' \wedge \forall \mathit{Char}_i \in \Upsilon\ .\  D(e, e'') =
0_{\mathit{diff}_i}]\}$.
\end{ax}

The above two axioms though necessary are not sufficient to
establish natural selection since these cannot be use as such to
distinguish between natural selection with neutral selection
\cite{SS00}. The following axioms are therefore needed to
sufficiently establish natural selection. \\

\begin{ax}[\textbf{Heritable Variation}] There must be variation in
heritable mutations in population of $\Lambda$. Formally, let
$$Child_{mut} = \{e \in \Lambda | \exists e' \in \Lambda\ .\  (e, e')
\in \mathbf{Parent_{\Delta}} \wedge [\exists \mathit{Char}_i \in
\Upsilon\ .\  0_{\mathit{diff}_i} \prec D(e, e')[i]]\}$$ be the
set of child entities carrying reproductive mutations. Let
$Var\_Child_{mut} \subseteq Child_{mut}$ be the set of those child
entities which carry different mutations with respect to
characteristics in $\Upsilon$, that is, $$\forall e, e' \in
Var\_Child_{mut} \textit{ we have } \exists \mathit{Char}_i \in
\Upsilon \ .\  0_{\mathit{diff}_i} \prec D(e, e')[i]$$ Then axiom
of heritable mutation demands that $|Var\_Child_{mut}|$ $\gg 1$,
that is, there are significantly many child entities carrying
different mutations. \\
\end{ax}

\begin{ax}[\textbf{Correlation}] There must be non zero correlation between
heritable variation and differential rate of reproduction.
Formally,
\begin{equation*}
\begin{split}
\forall \mathit{Char}_i \in \Upsilon\ .\ & \forall e, e' \in
Var\_Child_{mut}\ . \  \textit{the following two conditions should hold:}\\
                  & \begin{split}
                      i)\ e[i] <_i e'[i] \Leftrightarrow & [Rate_{rep}(e) < Rate_{rep}(e')]\ \vee\ [Rate_{rep}(e) > Rate_{rep}(e')]
                  \end{split} \\
                  & ii)\ e[i] =_i e'[i] \Leftrightarrow Rate_{rep}(e) = Rate_{rep}(e')
\end{split}
\end{equation*}
\end{ax}
Informally, this means as the value of characteristics inherited
by the child entity changes, rate of reproduction also changes.
Based upon the environmental pressures with respect to a
particular characteristics, rate of reproduction might either
increase or decrease as the characteristic changes.

The last two axioms state that there must be significant variation
in population (in characters) of entities which must be maintained
for evolutionarily significant periods and that this variation
must be caused by the differences in inheriting mutations from the
parent entities, which in turn directly affect the rate of
reproduction.

Having formalized the fundamental component of evolutionary
processes to be observed in a model, we will illustrate the
framework on two important ALife models in the following Section.
These illustrations will later be used in concluding
Section~\ref{sec:concluding} to extract generic design principles
for ALife research.

\section{Case Studies} \label{sec:case-studies}

\subsection{General Considerations}

Having described the generic formal framework in
Section~\ref{chap:framework}, which formalizes the  concept of
observations and consequent axiomatic inferences to establish the
level of evolution for ALife studies, in the following sections,
we will apply the formalism to different models as case studies.
These case studies include Cellular Automata based Langton Loops
\cite{Langton87} and $\lambda$ Calculus based Algorithmic
Chemistry \cite{ac:Fon92}. The case studies elaborate the steps
and technical details specific to the example universe of the
model, which remained implicitly defined in the generalized
description of the framework.

For a given model, the steps to instantiate the framework can be
described as follows: The observation process works on the
simulations of the model which iteratively change the underlying
states based upon the application of the updation rules of the
model. The observation process starts with the identification of
states of the model ($\Sigma$) during its simulations (i.e., state
sequences $\mathcal{T})$). Usually any change in the model (i.e.
the changes in the set of basic units) may give rise to a change
of the observed state. It is important to note that in some cases
there might be any changes in the observable state of the model
even tough there is ongoing underlying activity in the model, that
is, when model reaches, for example, a fix point.

For every state in the state sequence, the observation process (or
the observer) needs to identify a set of well defined entities
with suitable tagging for individual identification ($E$). These
entities need to be described in terms of their characteristics
($\Upsilon$). Next important task is to define the limits on the
observable mutational changes in individual characteristics of the
entities ($\delta_{mut}$, $\delta_{rep\_mut}$), which will in turn
define the recognition relation ($\mathbf{R_{\delta_{mut}}}$) to
relate entities persisting across states of the model as well to
determine whether two entities might be considered related under
descendent relationship.

Once the sets of entities in various successive states of the
model as well as their characteristics are known, important
evolutionary relationships need to be established between them.
These evolutionary relationship depend upon the intermediate
causal relation ($C$) between the entities as observed under the
mechanics of observation process. Using the limits on mutational
changes as well as causal relationship between entities, we
proceed to define the Ancestor ($\mathbf{AncestorOf}$) and the
Parent sets ($\mathbf{Parent_{\Delta}}$). These sets determine
whether there are entities which might be potentially reproducing
in the model, even with observable changes between parent and
child entities ($\Delta$).

Next stage of the observation process is to ascertain the level of
effectiveness of evolution in the model. Using the long term
observations on the model for statistically large number of
generations, one can infer some statistical patterns for degree of
heredity and variation. For natural selection to be effective,
there should exist large number of reproducing entities with
significant variation in their characteristics such that there
exists correlation of this variation in the characteristics with
the reproductive success of the entities.

This process at the end establishes the validity of all or some
axioms of the framework for the given model which provides clues
to the degree upto which evolutionary processes might be effective
in that model universe. The case studies in the following sections
will illustrate this process in detail.

In these case studies, constructs not explicitly defined are
assumed to be same as what is defined in the framework.


\subsection{Case Study \bf{1}: Langton Loops}
\label{chap:langton}

%

Research on the self reproduction has a long cherished history
starting in early fifties \cite{Burks70,Sipper98,Freitas04}. After
the pioneering work of Alan Turing in early 40s to define the
mechanical meaning of `computation' as a Turing machine
transitions, John von Neumann defined Cellular Automata
(CA)~\cite{Neumann66} to explain the generic logic of self
reproduction in mechanical terms. His synchronous cellular
automata model was a two dimensional grid divided into cells,
where each cell would change in parallel its state based upon the
states of its neighborhood cells, its own state and its transition
rule. For such CA model, von Neumann defined a virtual
configuration space where he demonstrated analytically that there
exists some universal replicator configuration which could
replicate other configurations as well as itself. Though universal
replicators are not found in nature and such self replicator was
extremely large in its size, nonetheless the underlying logic of
treating states of cells in the grid both as `data' as well as
`instruction' was very fundamental contribution of this model and
that was exactly was was discovered later in case of real life
where DNA sequences specify both transcription as well as
translation for their own replication in a cell. Another strength
of von Neumann's formulation was its ability to give rise to
unlimited variety of self replicators
\cite{McMullin00a,McMullin00b}. Over the years this model was
simplified and reduced in size considerably \cite[81-105]{Codd68}.

Finally Langton introduced loop like self replicating structures
in~\cite{ac:Lan84}, which retained the `transcription -
translation' property of von Neumann's model excluding the
capability of universal replication and symbolic computation.
Langton's original self-replicating structure is a 86-cell loop
constructed in two-dimensional, 8-state, 5-neighborhood cellular
space consisting of a string of core cells in state $1$,
surrounded by sheath cells in state $2$. These loops have since
then, been extended into several interesting directions including
evolving Evoloops in \cite{Samaya98b}.

These cellular automata based ALife models offer the ideal example
for our observer (observation process) based framework since these
replicating loops and their variations evolve only with respect to
some high level observation process, which can be used to define
entities (loops) and their evolution. We will illustrate the
formal framework by instantiating it on the Cellular automata
based Langton loop model. Further details on the model itself can
be found in the above references.

\subsubsection*{Instantiating the Framework}
\label{sec:langton_inst}

We consider the case of two dimensional CA lattice based model.An
observation is defined on the CA model by assuming an underlying
coordinate system such that each cell in a two dimensional
cellular automata (CA) lattice can be associated with unique
coordinates (represented as $(x, y)$.) A cell is then completely
represented as $\lan(x, y), s\ran$, where $s \in [0..7]$ is the
state of the cell. When a cell is in state $0$, it is also known
as a \emph{quiescent} cell. Let us denote the set of all cells of
a CA model as $Cell$, which is a potentially infinite set.

For a given cell $\lan(x, y), s\ran \in Cell$, its coordinates can
be accessed as follows: $co_x(\lan(x, y), s\ran) = x$,
$co_y(\lan(x, y), s\ran) = y$, which can be extended to the set of
cells: $\forall Z \subseteq Cell$, $co_x^+(Z) = \bigcup_{c \in Z}
co_x(c)$, $co_y^+(Z) = \bigcup_{c \in Z} co_y(c)$.

$Neigh: Cell \rightarrow 2^{Cell}$ gives the coordinate wise non
quiescent cells in the surrounding neighborhood of a cell.
Formally, $\forall (c = \lan(x, y), s\ran) \in Cell$ we have
\begin{equation*}
Neigh(c) = \{ \lan(x \pm 1, y), s'\ran, \lan(x, y \pm 1), s'\ran\
|\ s' \neq 0\}
\end{equation*}

\subsubsection*{The model Structure}

A CA-based model is usually initialized by setting some finite
number of selected cells to non-quiescent states. At each step,
state of every cell of the model is changed as per the state
transition rules. Therefore we define for an observer \emph{state}
of the Langton's model as the subset of $Cell$ consisting of only
non quiescent cells. It is clear that for the observer change in a
state is observable only if there is a change in the set of non
quiescent cells. The \emph{state} of the model for the observer
will also be referred to as $configuration$. Thus $\Sigma$ denotes
the set of all possible different configurations and a state
sequence in $\mathcal{T}$ is a sequence of configurations observed
in temporal order by the observer starting from some specific
configuration. In the following discussion we will consider a
fixed sequence given as $T \in \mathcal{T}$, starting with a
specific initial state given in Figure~\ref{ftr1} (Time 0).
For the fact that there exist a temporal (total) ordering of
states in $T$, we can also associate an integer sequence $I = [0,
1, 2, \ldots]$ with $T$, which works as an indexing for the
states. With the above structure of Langton's CA model, the
observer takes the following decisions.

\subsubsection*{Entities}

Each entity in some state is characterized by two values - the
connected set of non quiescent cells and the associated
\emph{pivot}. Two cells are connected only if there exists a
consecutive sequence of neighboring non quiescent cells joining
them in the lattice. The (function) \emph{pivot} gives the
coordinates for a cell uniquely associated with an entity in CA
lattice in a particular state. Formally, the set of entities
(loops) in the model is defined as follows:
\begin{equation*}
\begin{split}
E = \{ & [Z, pivot(Z)] \mid \exists\  \textit{a configuration }
          S \in T\ .\  \\
       & [Z \subseteq S \wedge Z \neq \emptyset] \wedge
       [\forall c \in Z\ .\  \exists c' \in Neigh(c)\ .\  c' \in Z] \}
\end{split}
\end{equation*}
To define $pivot$, an observer may choose the coordinates of top
left hand corner cell of an entity as the pivot for it. Formally
$$pivot(Z) = (min\{co_x^+(Z)\}, max\{co_y^+(Z)\})\ \forall (e  =
[Z, pivot(Z)]) \in\ E$$ This gives an obvious characterization for
a two dimensional character space $\Upsilon = \mathit{Char}_1
\times \mathit{Char}_2$ with $\mathit{Char}_1$ being the set of
all non quiescent connected set of cells and $\mathit{Char}_2$
being the set of corresponding pivots. We do not associate
additional tags with entities because pivots can be used to
uniquely identify them in any state of the model.

%
\subsubsection*{State Function}
$F: E \mapsto I$ is defined using a table which associates with
each entity $e \in E$, the index $i \in I$ for the state in which
$e$ is observed.
%


\subsubsection*{Distance Measure}

Distance function $D: E \times E \rightarrow \{0, 1\} \times \{0,
1\}$ is defined such that $\forall e, e' \in E\ .\ D(e, e') =
[d_g, d_p]$ where $d_g$ and $d_p$ are defined as follows: $d_g$ is
$0$ only if both entities have the same number of cells arranged
identically or else it returns $1$. 
$d_p$ is $0$ when the pivots for both the entities are same and
$1$ otherwise.

\subsubsection*{Limits on Observable Mutations}

The observer next selects $\delta_{mut} = [1, 0]$, which means
that observer can recognize an entity in future states even with
mutations (changes in the states, number, or the arrangement of
cells comprising the entity) provided that the pivot remains the
same. Select $\delta_{rep\_mut}= [0,1]$ which implies that for
reproduction observer strictly demands identical geometrical
structure of the parent and child entities, though may have
different pivots - this is essential to capture exact replication
of the loops.

\subsubsection*{Observing Reproduction and Fecundity}

\textbf{Recognition relation} $\mathbf{R_{\delta_{mut}}}$ $: E
\rightarrow E$ is defined as follows:
\begin{equation*}
\forall e, e' \in E, \mathbf{R_{\delta_{mut}}} (e) = e'
\Leftrightarrow\ [F(e') = F(e) + 1] \wedge [D(e, e') \leq
\delta_{mut}]
\end{equation*}
Informally this means two entities in consecutive states are
recognized same only if they have the same pivots. Which also
means observer can recognize entity even with change in the
number, state, and geometrical arrangement in the cells of an
entity across states provided that entity does not shift in CA
lattice altogether (which would result in the change of the
pivot.)\\

\begin{lemma} $\mathbf{R_{\delta_{mut}}}$ satisfies Axiom $1$, Axiom
$2$, and Axiom $3$. \end{lemma}

\begin{proof}
Axiom $1$ and Axiom $3$ are satisfied by definition. Axiom $2$,
which states that $\mathbf{R_{\delta_{mut}}}$ is an injective
function holds because no two entities in the same state share the
same pivot. This is because pivot as defined above is connected to
all other cells of the entity and all the non quiescent cells
which are connected in any state are taken together as one entity.
Thus two different entities in the same state always consist of
cells such that cells in one entity are not connected with the
cells of second entity, and hence always have different pivots.
\end{proof}


\textbf{Causal relation} The relation $C$ between entities in
consecutive states is defined as follows: $C \subseteq E \times E$
such that $\forall\ e, e' \in E$ where $e = [Z_e,
\mathit{pivot}(Z_e)]$ and $e' =[Z_{e'}, \mathit{pivot}(Z_{e'})]$
we require
\[
(e, e') \in C \iff \left\{ \begin{array}{ll} 1.\, & co_x^+(Z_e) \supset co_x^+(Z_{e'})\\
2.\, & co_y^+(Z_e) \supset co_y^+(Z_{e'})\\
3.\, &  \mathit{pivot}(Z_e) \neq \mathit{pivot}(Z_{e'})\\
4. & F(e') = F(e) + 1
\end{array} \right.
\]
Intuitively what we demand with above definition of causal
relation $C$ is that child entity was part of the parent entity
and at certain stage it ``breaks off'' from the parent
entity, as can be seen in Figure~\ref{ftr1} at time step $127$.\\

\begin{lemma} Causal relation $C$ defined above satisfies the Causality Axiom. \end{lemma}

\begin{proof}
Condition $F(e') = F(e) + 1$ insures that $e$ and $e'$ are not
observed in the same state. To establish that $e'$ is not the
result of mutations in some other entity $e''$ observed in past
(i.e., $[F(e'') = F(e)] \wedge [\mathbf{Rec}(e'') = e']$) we note
that because of the definition of $\mathbf{Rec}$, $e''$ and $e'$
would otherwise have the same pivots, which means pivot of $e''$
will be included in the set of cells in $e$ (since $[co_x^+(Z_e)
\supset co_x^+(Z_{e'})] \wedge [co_y^+(Z_e) \supset
co_y^+(Z_{e'})]$), which is not possible because $e$ and $e''$
being different entities in the same state cannot have cells in
common including pivot as argued above in the proof of previous
lemma.
\end{proof}
%

\begin{figure}[hbtp]
\begin{center}
\includegraphics[scale=0.8,trim=0 0 0 0,clip=]{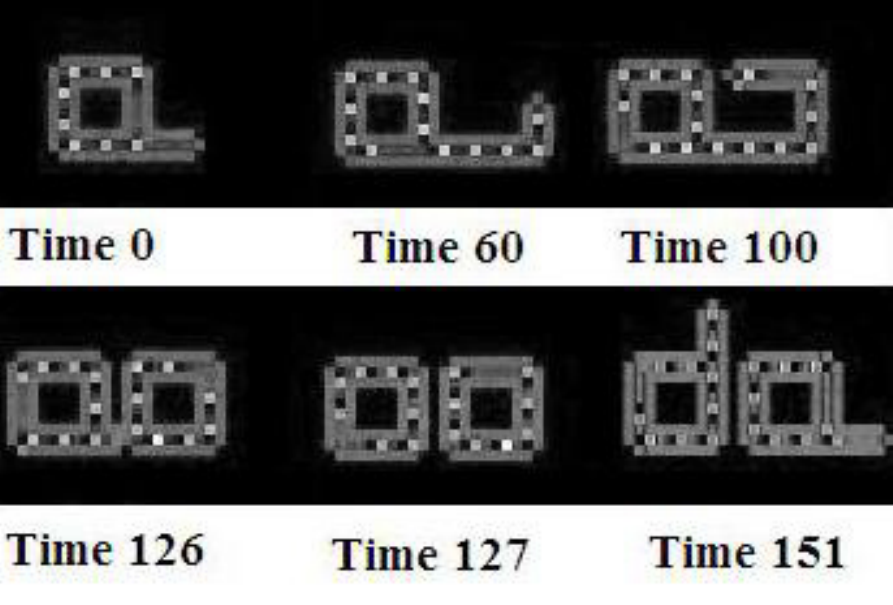}
\caption[Self-replication in Langton loops]{Self-Reproduction in
Langton loops; screen shots from~\cite{Sayama2005}} \label{ftr1}
\end{center}
\end{figure}

\begin{lemma} \textbf{Axiom of Reproduction} and the
\textbf{Axiom of Fecundity} are satisfied by the entities and
abstractions on Langton Loops described above. \end{lemma}

\begin{proof}
These two axioms can be established by the observer in a specific
state sequence as exemplified in Figure~\ref{ftr1} and
Figure~\ref{ftr2} by repeatedly applying the recognition relation
$\mathbf{Rec}$ when entities are changing in number and states of
cells (retaining the pivots) and applying the causal relation when
a parent entity splits (e.g. at Time=$127$). The relation $\Delta$
connects the initial parent entity and the child entity at
Time=$151$.

With respect to Figure~\ref{ftr1},  an entity is identified at
Time=$0$ with associated pivot. Between time steps $[1\ldots126]$
 entity changes in number and states of its cells but the pivot
remains the same, hence as per the definition of $\mathbf{Rec}$,
the observer can recognize the entity in these successive states.
At Time=$127$, the (parent) entity is observed to be splitting
into two identical copies. One of these is again recognized as the
original parent entity because of its pivot and the second entity
would be claimed to be causally related with the parent entity as
per the definition of $C$. To see this, notice that the parent
entity at Time=$126$ contains all the cells of the child entity
appearing at Time=$127$, which satisfies the definition of $C$.
Between time steps $128$ and $151$ both parent and child entities
undergo changes in the number and states of their cells but their
pivots remain fixed. Hence they can again be recognized. Finally
at Time=$151$ the child entity becomes identical to the original
parent entity, therefore the parent entity at Time=$0$ and the
child entity at Time=$151$ are related using $\Delta$. The
transitive closure finally give us the final descendence
relationship between the parent and the child entity.
\end{proof}

\begin{figure}[hbtp]
\begin{center}
\includegraphics[scale=0.7,trim=0 0 0 0,clip=]{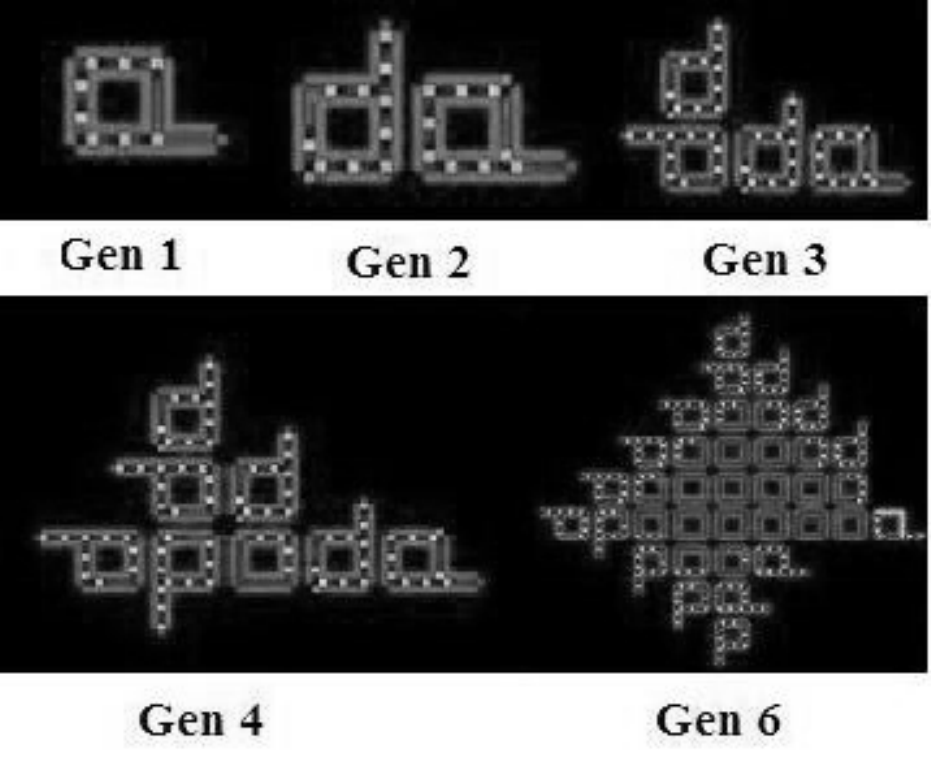}
\caption[Fecundity in Langton loop population]{Fecundity across
generation in a population of Self Replicating Langton Loops;
screen shots from~\cite{Sayama2005}} \label{ftr2}
\end{center}
\end{figure}

\subsubsection*{Mutations, Inheritance, and Natural Selection}

Primary focus of Langton while defining the CA based replicating
loop model was to demonstrate that genotype - phenotype based
coding decoding scheme can be captured in CA universe as well
\cite{Langton87}. And we have seen that this can be observed by
the observer as defined above. Nonetheless, Langton loops do not
exhibit mutations and indeed if we analyze the underlying state
transitions defined for the cells in the model, it becomes clear
that the transition behavior required for the reproduction changes
immediately if any changes are introduced in an entity and
resulting entity is no longer capable of reproduction or in other
terms, none of the mutations in existing replicating loops
preserve reproduction and in terms of the current framework
\emph{Axiom} of \emph{Preservation of Reproduction under
Mutations} is not valid. Because of the enormity of possible
configurations and transition dynamics it is not easy to analyze
which kind of replicating loops can ever withstand certain
mutations and can preserve replicating functionality. Heredity of
course is worth considering only when entities mutate and continue
reproduction. Thus with existing Langton loops, an observer cannot
observe heredity and subsequent natural selection.

The extension of Langton loops defined by Sayama as
\emph{Evoloops} in~\cite{Samaya98b} is one such attempt, where not
all the loops in the model are of the same type with respect to
the number and geometrical arrangement of cells and final
population witnesses (small) variety of different kinds (in size)
of reproducing loops scattered on the lattice forming colonies.
The Evoloops and their evolution can be formulated in the
framework by suitably modifying the definition of the distance
measure $D$ to measure the differences between the entities in the
number and geometric arrangement of cells and by changing limit
$\delta_{rep\_mut}$ such that the observer is able to establish
descendence relationship even when the parent and the child
entities (loops) are not identical. Since evoloops of different
types replicate at different rates, where rate of replication is
measured in terms of number of state transitions, we can infer
that the loops satisfy the axiom of sorting. Indeed in a weak
sense with available simulation results it appears that evoloops
can be observed demonstrating heredity as well as selection.

\subsubsection*{Conclusion}

We have seen that we can formally define an observation process on
the CA universe which discovers the self replication of so called
Langton loops during the simulation of model. The specific
observer presented here follows the intuition that Langton
implicitly stated when describing the loops. We also noted that
mutations, heredity, and selection based axioms are not met in the
model where this limitation can be attributed to underlying
transition rules of the model. Evoloops, which were designed as
extensions of Langton loops with mutations can be seen to be
evolving with variation in the sizes and rates of reproduction.


\subsection{Case Study \bf{2}: Algorithmic Chemistry}
\label{chap:lambda}

Algorithmic Chemistry (AlChemy) was introduced in \cite{ac:Fon92}
and further discussed in
\cite{ac:FB94arrival,ac:FWB94,FB94twice,ac:FB96}. The main focus
of the AlChemy is to study the principles behind the emergence of
biological organizations with the approximate abstraction of real
chemistry as $\lambda$ calculus with finite reductions. Starting
with a random population of $\lambda$ terms (molecules), using
different filtering conditions on reactions, authors describe the
emergence of different kinds of organizations: \emph{Level $0$}
organization consisting of a set of self copying $\lambda$ terms
and hypercycles with mutually copying $\lambda$ terms, \emph{Level
$1$} \emph{self maintaining} organizations consisting of $\lambda$
terms such that every term is effectively produced as a result of
reaction between some other terms in the same organization and
lastly \emph{Level $2$} organization consisting of two or more
Level $1$ sub organizations such that molecules migrate between
these self maintaining sub-organizations. They also provide
detailed algebraic characterization of Level $1$ and Level $2$
organizations without referring to the underlying syntactical
structure of the $\lambda$ terms (molecules) or the micro dynamics
(reduction semantics and filtering conditions) governing the
output of reactions.

\subsubsection*{Instantiating the Framework}

In view of the proposed observer based framework, characterization
of self replicating molecules and hypercycles consisting of
mutually copying molecules is achieved by defining an observation
process, which focusses on individual $\lambda$ terms as entities
and identifies hypercycles as a set of individually replicating
$\lambda$ terms in a sequence of reaction steps (reflexive
autocatalysis).

Since Level $1$ and Level $2$ organizations emerge only when self
copying reactions are filtered out (i.e., self reproduction is not
allowed) to ensure that Level $0$ organizational structures do not
become the fixed points, these cannot be analyzed under the
current framework design because we only consider reproduction,
mutation, inheritance, and selection based evolution and emergence
of organizations.

\subsubsection*{The Chemistry Structure}

A chemical soup of AlChemy consisting of $\lambda$ terms as
molecules is usually initialized with a population of large number
of randomly generated $\lambda$ terms. A \emph{state} of the
chemistry could, therefore, be considered as the collection of all
these $\lambda$ terms (with multiplicity). Since every non elastic
reaction results into introduction of output $\lambda$ term into
the soup and possible removal of some other randomly chosen terms,
it is natural to consider such succession of states after every
reaction step as a state sequence $T \in \mathcal{T}$.

The components of the observation process defined next are based
upon the assumption that it is possible to observe the inputs
terms for a reaction (collision), resultant output term to be
added to the soup, and the randomly deleted terms from the soup,
without knowing the actual reaction details or the reduction
semantics.

\subsubsection*{Entities}

For a given state of the chemistry, let the observation process
identify each $\lambda$ term as a separate entity associating an
unique integer tag with it. Each such entity is represented as
$[w, i]$ where $i$ is the tag uniquely associated with $\lambda$
term $w$. $E$ is the set of all such entities in the chemistry.



{\sf Tagging}: Suitable tagging mechanism needs to be defined by
the observer to recognize whether two $\lambda$ terms in
successive states are the same and to distinguish between multiple
syntactically identical copies of a $\lambda$ term in the soup at
any state. We can associate tags of the form $\lan i_{size},
i_{lex}, i_{mul}\ran$ ($i_{size}, i_{lex}, i_{mul} \in N$) with
the individual molecules in the following way: for the initial
population of $\lambda$ terms, they are arranged with respect to
their sizes and we assign the size of these terms as the first
component in their tags ($i_{size}$) and for terms of same size
arrange them lexicographically and assign in increasing order
second component of their tags ($i_{lex}$) such that multiple
copies of a term have the same first two components of their tags
and then assign increasing integers to each of these as their
third component of the tag ($i_{mul})$. Under such tagging scheme
a small population of $\lambda$ terms $\{\lambda x. x, \lambda x.
x, \lambda x_1. \lambda x_2.x_2\}$ defines the state - $\{[\lambda
x.x, \lan 3, 1, 1\ran], [\lambda x. x, \lan 3, 1, 2\ran], [\lambda
x_1. \lambda x_2. x_2, \lan 5, 1, 1\ran]\}$. For a given tag $tg =
\lan i, j, k \ran$ its components are accessed as $i = tg[1], j =
tg[2]$, and $k = tg[3]$.

Next we discuss the mechanism for \emph{updating} these tags after
reaction and elimination steps. 
We increment by one the third component of the tags for each
entity , which was not deleted from the soup from previous state
and give new unique tag to the new terms added to the soup with
respect to their position in the list of terms based on their size
and lexicographic order such that third component of the newly
added terms is always given value $1$. This numbering scheme
reliably maintains the recognition of terms across states of the
chemistry.

\subsubsection*{Distance Measure}

Distance function $D: E \times E \rightarrow \{0, 1\} \times \{0,
1\}$ is defined such that $\forall (e=[w, t_g], e'=[w', t'_g]) \in
E\ .\ D(e, e')[1] = 0$ if $w$ and $w'$ are the same with respect
to $\alpha$ renaming implying that entity $e'$ is the same entity
$e$ in the previous state; otherwise $D(e, e')[1] = 1$. $D(e,
e')[2] = 0$ if $t'_g[3] - t_g[3] = 1$ indicating that entity $e$
is observed in the next state as entity $e'$, otherwise $D(e,
e')[2] = 1$. The distance function $D$ has been defined keeping in
mind the use of these distances in defining recognition relation
later.

\subsubsection*{The Limits on Observable Mutations}

Let $\delta_{mut} = [0, 0]$, indicating that syntactically
different $\lambda$ terms (under $\alpha$ renaming) are treated as
different entities. Also let $\delta_{rep\_mut} = [0, 1]$
indicating that reproductive mutations resulting into
syntactically different term are not observable. This is primarily
because under $\beta$ reduction semantics of Alchemy, even changes
in the syntactical representations result into very different
reaction behaviors.

\subsubsection*{Observing Self Replicating Hypercycles}

\begin{figure}
\centering
\includegraphics[scale=0.6,trim=0 0 0 0,clip=]{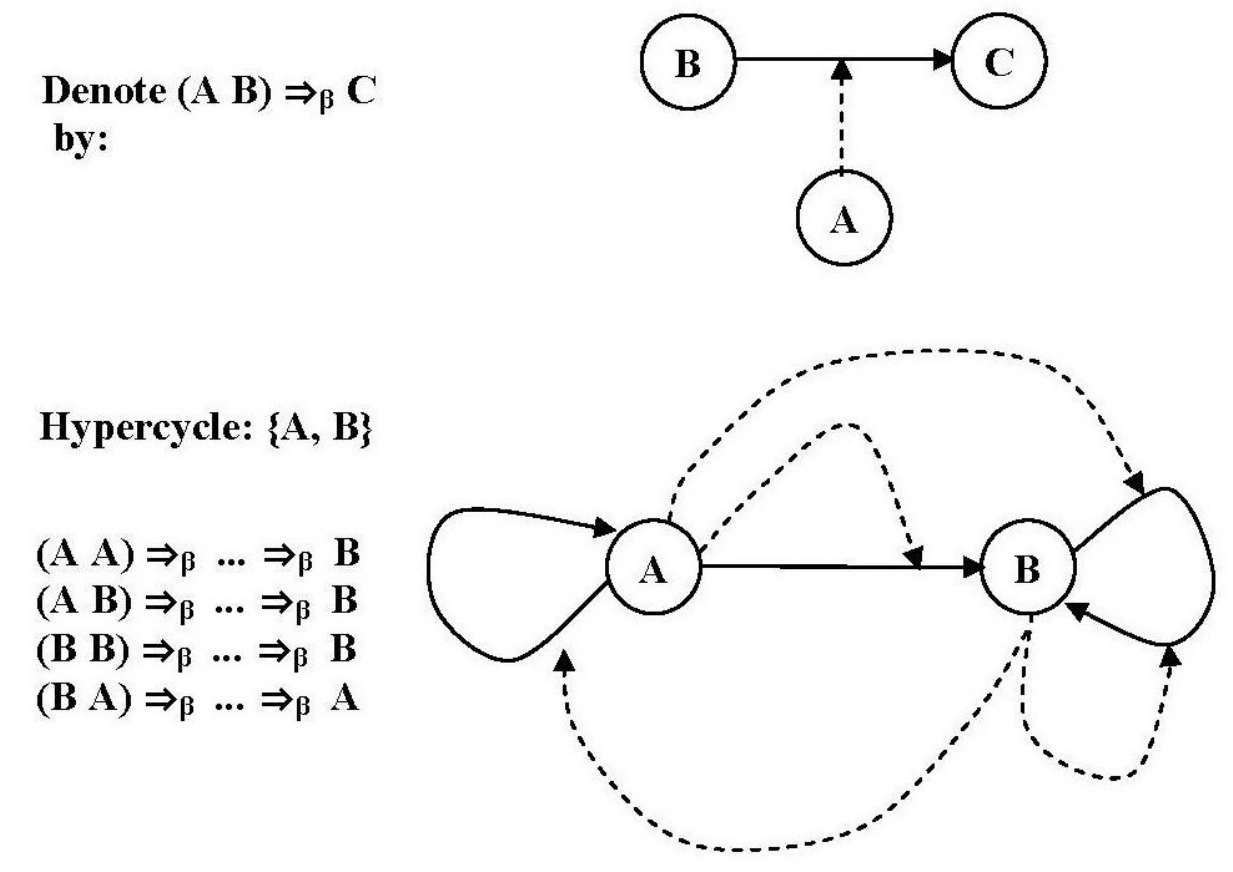}
\caption[An example of self replicating elementary hypercycle in
AlChemy]{Example of self replicating elementary hypercycle
organization in AlChemy from \cite{ac:FB94arrival}. $A = \lambda
x_1. \lambda x_2. x_2$ and $B = \lambda x. x$ . $(A B)
\Rightarrow_\beta C$ represents reaction between $A$ and $B$ by
applying $A$ on $B$ yielding $C$ under $\beta$ reduction.}
\label{fig:lambda_figure_aj}
\end{figure}


We can observe the self-replicating elementary hypercycles as sets
of self-replicating entities. Let us define, for that purpose, the
recognition relation $\mathbf{R_{\delta_{mut}}}$ $: E \rightarrow
E$ as follows: $\forall e, e' \in E$, $\mathbf{R_{\delta_{mut}}}$
$(e) = e' \Leftrightarrow [F(e') = F(e) + 1] \wedge D(e, e') \leq
\delta_{mut}$. Informally this means two entities in consecutive
states are recognized same only using their tags.\\

\begin{lemma} $\mathbf{R_{\delta_{mut}}}$ satisfies Axiom $1$, Axiom
$2$, and Axiom $3$. \end{lemma}

\begin{proof} Axiom $1$ and Axiom $3$ are satisfied by definition.
Axiom $2$, which states that $\mathbf{R_{\delta_{mut}}}$ is an
injective function holds because of the specific construct of
tagging mechanism and the definition of Distance function $D$
which is such that two entities in successive states are
recognized as same only when the difference between their third
components of tags is $1$, and we know that the observer selects
new tags in such a way that this difference is $1$ only when same
entity was present in the previous state.
\end{proof}

Next let us defines $\Delta \subseteq E \times E$ such that
$\forall e, e' \in E . (e, e') \in \Delta \Leftrightarrow D(e, e')
\leq \delta_{rep\_mut}$. In order to define causal relation
between entities in the AlChemy, we assume that observer has the
knowledge of the reacting entities and the output term at any
state. Therefore if entities $e_1$ and $e_2$ react in some state
and yield $e_o$, the observer defines causal relation $C$ so that
$(e_1, e_o) \in C$ and $(e_2, e_o) \in C$ with $F(e_1) = F(e_2) =
F(e_o) -1$.\\

\begin{lemma} Causal relation $C$ defined above satisfies Axiom
$4$. \end{lemma}

\begin{proof}
First condition of Axiom $4$ is satisfied by definition since
$F(e_o) = F(e_1) + 1 = F(e_2) + 1$.  The second condition $[\not
\exists e' \in E. F(e_1) = F(e') \wedge \mathbf{R_{\delta_{mut}}}$
$(e') = e_o]$, that is, there does not exist any third entity $e'$
in the previous state, which has mutated into $e_o$, again follows
from the specific construct of tagging as well as the distance
function because as per the tagging mechanism explained before
$e_o$ being newly added entity in the chemistry will have the
$3^{rd}$ component of its tag as $1$ and all previously present
entities, including $e_1, e_2$, in the chemistry would have their
tags in new states updated such that their $3^{rd}$ components are
always greater than $1$.
\end{proof}

Relations $\mathbf{AncestorOf}$ and $\mathbf{Parent}$ can be
defined same as in the framework.\\

\begin{lemma} \textbf{Axiom of Reproduction} and the
\textbf{Axiom of Fecundity} are satisfied by the entities and
corresponding abstractions discussed above.
\end{lemma}
\begin{proof}
These two axioms depend upon the examples of self replicating
$\lambda$ terms as well as elementary hypercycles. In case of
hypercycles, the observer establishes multi-step reproduction
using transitive closure of causal relation for each of the
entities in the hypercycle. A quite well known example of self
replicating $\lambda$ term is $\lambda x. (x)(x)$ since $(\lambda
x. (x)(x))(\lambda x.(x)(x))\ {\Rightarrow}_{\beta}\ (\lambda x.
(x)(x))(\lambda x. (x)(x))$. Though in case of Alchemy, the level
$0$ organization consists of self-copiers like $\lambda x. x$ and
hypercycles like $\{\lambda x_1. \lambda x_2. x_2, \lambda x. x$\}
as illustrated in Figure 3. As per the definition of causal
relation, entity instances of $\lambda x_1. \lambda x_2. x_2$ and
of $\lambda x. x$ are causally related to past instances of each
other and therefore of themselves.
\end{proof}

\subsubsection*{Mutations, Inheritance, and Natural Selection}

As emphasized in \cite{ac:FB94arrival}, primary goal of AlChemy is
to study alternative pathways in which higher level organizations
(i.e., hypercycles, self maintaining organizations) can emerge
starting with a random set of molecules. Therefore it appears that
there is no explicit notion of mutations present in the chemistry.
To see this notice that every new entity in the population is the
result of reaction between two other entities. Therefore if one
particular observer decides that one of the reacting entities is
mutating into the resulting entity, it is still difficult to
decide which of the two reacting entities should be considered as
mutating into the new one. Even if such a view is adopted, the
observer will observe that if a self-copying entity at any
reaction step mutates into another entity then most often the new
entity can no longer self-copy. Thus \emph{Axiom} $7$
(Preservation of Reproduction under Mutation) would be violated.
Finally as discussed at the beginning of the section, owing to the
focus of our framework on the evolutionary processes,
self-maintaining organization of the kind that arise in AlChemy
are beyond the scope.

\subsubsection*{Conclusion}

Thus we have demonstrated that, based upon the knowledge of
reacting terms and outputs, a precise observation process can be
defined to work with AlChmey, which can be used to discover the
self replicating $\lambda$ terms as well as hypercycles in the
model. We also noted that mutations, heredity, and selection based
axioms are not met in the chemistry where this limitation should
be attributed to underlying reaction semantics
of the chemistry as well as its design. 
This study highlights the fact that not all interesting dynamic
processes are evolutionary in nature and therefore some of these
non evolutionary processes are out of scope of the framework at
present.


\section{Related Work} \label{chap:related}

Because of the presence of sufficiently many biology-specific
criterion (e.g., morphological characters, bio-molecular
structures etc.) to distinguish life from non-life, in biological
literature there is little formal work on recognizing life {\it
per se}. There is, however some recent work on defining and
developing methods to analyze genotype space structure based upon
the macroscopic observations on phenotype characteristics (mainly
morphological and reactive
characteristics)~\cite{Garay98,Gamez03,Lopez04}.

To the authors' knowledge, there is not much work focussing on the
observation process for ALife studies reported in literature.
Though there exist proposals to define `numerical parameters' or
`statistics' \cite{Bedau99} to recognize life in a model. However,
it is not clear whether there can be simple numerical definitions
capturing the essence of life in arbitrary models and even if so
does not seem to be the case with the existing proposals. The
difficulty arises out of intricate nature of reproduction and
selection inevitably involving non trivial identification of the
population of evolving entities. Langton defined
in~\cite{Langton91} a quantitative matric, called \emph{lambda}
parameter to detect life in any generic one dimensional cellular
automata model based upon the characteristics of its transition
rules. This lambda parameter based analysis is based upon the
assumption that any self organizing system can be treated as
living and does not consider population centric evolutionary
behavior as characteristic of life. In \cite{Bedau98} there is a
discussion on the classification of long term adaptive
evolutionary dynamics in natural and artificially evolving
systems. This they achieve by defining activity statistics for the
components, which quantifies the adaptive value of components
(characteristics in our model). They employ similar mechanism as
of ours by associating activity counters (tags) with all the
components present in the system during simulation.


Self-reproduction, which has a long history of research starting
from the late 1950s~\cite{Burks70,Sipper98,Freitas04} has evaded
precise formal definition applicable to a wide range of
models~\cite{ND98} in the sense of observable characterization of
the reproducing entities. Though there is enough work on
mathematical analysis of replication dynamics (fecundity) in
various natural systems or the systems where environmental
constraints governing the rate of reproductions are known (see for
overview~\cite[Chap5]{Freitas04}.) In some of the discussions
related to self-replication in cellular automata
models~\cite{Samaya98b,Morita98}, formalizations of reproducing
structures are presented, but they do not attempt to provide a
general framework for observing reproduction or other components
of evolutionary processes. These attempts at formalizing
reproduction in CA models are reminiscent of our definition of
entities (loops) in Section~\ref{sec:langton_inst}.

In other work~\cite{Misra06}, we proposed a multi-set theoretic
framework to formalize self reproduction (with mutations) in
dynamical hierarchies in terms of hierarchal multi-sets and
corresponding inductively defined meta-reactions. The ``self" in
``self-reproduction" was defined in terms of {\it observed
structural equivalences} between entities. We also introduced
constraints to distinguish a simple ``collection" of reacting
entities from genuine cases of ``emergent" organizational
structures consisting of {\it semantically coupled} multi-set of
entities.


\section{Conclusion} \label{sec:concluding}

\subsection{General Remarks}

This paper formalizes an implicit underlying component of ALife
studies, namely the observation process, by which entities are
identified and their evolution is observed in a particular ALife
simulation. Under the assumption that the essence of life-like
phenomena is their evolutionary behavior, we developed a framework
to formally capture basic components of evolutionary phenomena.
This work, in essence, brings insights from evolutionary theory
for real-life into the realm of artificial-life for defining a
formal framework for observational processes, which are needed for
the identification of life-like phenomena in the ALife studies. We
have argued that without such a formalism, claims pertaining to
the evolutionary behavior in ALife studies will remain
inconclusive.

We formally elaborate in algebraic terms the necessary and
sufficient steps for an observational process, to be employed by
an ALife researcher upon the time progressive model of his model
universe, to uncover (hidden) life-like phenomena in the light of
Darwinian evolution as defining characteristics of life. The
observation process as specified in our framework may be carried
out manually or can be alternatively algorithmically programmed
and integrated within the model.

To define inference process we specify necessary conditions, as
axioms, which must be satisfied by the outcomes of observations
made upon the model universe in order to infer whether life-like
phenomena is present in the model
(Section~\ref{chap:evocomponents}). These axioms also specify the
experimental work necessary in order to observe and lay claims for
the presence of life in the model universe. 

The case studies on Langton loops (Section~\ref{chap:langton}) and
Algorithmic Chemistry (Section~\ref{chap:lambda}) highlight the
contributions that such an approach can make to the discussion of
specific ALife experiments. An important property of such a study
is to make explicit ``multi-level observations", where entities
and their relationship can be observed and defined on separate
organizational levels.

The framework design and the case study analysis also provide us
clues for ALife research designs so that to be better able to
witness evolutionary phenomena in the model during its
simulations. This is discussed next:

\subsection{Design Suggestions for ALife Researchers}
\label{sec:design-suggst}

As the framework is based upon the Darwinistic concepts of
defining life in terms of evolutionary processes, the design
suggestions we describe here are rather more suitable for those
studies which aim to complement real life studies in an
evolutionary framework.

\begin{itemize}
    \item \textbf{Sufficient Reproduction with Variation:} The model
    must be designed such that there exist potentially large set
    of reproducing entities with significant variation in their
    characteristics. Quite often this hinges upon the choice of reaction rules or
    the semantics of the model and indeed it is a serious challenge for any
    model designer to define the reaction semantics which
    permits potentially large set of reproducers with significant variation. Another
    interesting aspect is that these reproducers must be relatively closely
    related to each other under the reaction semantics. This means that
    sufficiently many variations of reproducers should also be reproducers in
    themselves otherwise the axiom of preservation of reproduction
    under mutation will not effectively hold in the model and most of the
    reproducers would have to appear \emph{de novo} during simulations. We
    encounter this problem in both of the case studies
    discussed in Section~\ref{sec:case-studies}. In case of Langton loops, any kind of
    change in the loop structure would cause caseation of
    replication. The work on designing Evoloops is therefore based upon the redefinition of
    the reaction semantics or transition rules which permit variation in replicating loops.
    Similarly in the case of Algorithmic chemistry, almost all of the single replicating $\lambda$
    terms arise \emph{de novo} and their variations do not replicate under $\beta$ reaction
    semantics.
    \item \textbf{Measurable Rates of Reproduction:} The model should be designed
    such that it is possible to impose some valid measure of
    determining the rates of reactions which in turn can be used
    to estimate differences in the rates of reproduction of different entities.
    This measurement of reproductive rates must be independent of the updation
    algorithm which selects entities for reaction. Therefore
    it can be argued that the models, where all (reproductive) reactions
    take place in a single step would be difficult to observe for
    natural selection, which works only when different entities
    reproduce at different rates.  For example, it is not
    possible to infer differences in the rates of reproduction
    among different reproducing elementary hypercycles in the
    Algorithmic Chemistry consisting of
    the same number of $\lambda$ terms because every
    reaction between any two $\lambda$ terms occurs in a single
    step. On the other hand natural selection can be observed in case of
    Evoloops precisely because different types of loops
    consisting of different number of cells reproduce at different
    rates based upon the number of state transitions.
\end{itemize}

\subsection{Limitations}
\label{sec:limit}

The decision to equate life with evolutionary processes also
excludes some of the interesting complex phenomena that are not
evolutionary in nature from the scope of this work. Indeed, we
have shown in Section~\ref{chap:lambda} that the framework cannot
account for the dynamic non-evolutionary behavior of Level $1$ and
Level $2$ organizations emerging in the Algorithmic Chemistry. We
limit our attention to only those observations having evolutionary
significance, though other observations can also be made upon the
model including metabolism~\cite{ac:BFF92}, emergence of
complexity~\cite{ac:adami00}, self organization~\cite{ac:Kau93},
and autonomous and autopoitic nature of life~\cite{Zeleny81} etc.

We have not placed direct emphasis on certain concepts widely
associated with ALife studies including the notion of
``emergence''. In our current setting the notion of ``strong
emergence'' is only implicitly present and indeed ``the element of
surprise'' \cite{Bass97} often associated with emergence is not
immediate in the framework. Similarly ``the element of autonomy''
of emergent processes with respect to the underlying micro-level
dynamics is not addressed in our framework. Indeed, the spirit of
the high level of observations and corresponding abstractions upon
which the framework rests, may preclude such inferences.
Nonetheless the idea of ``weak emergence" \cite{Bedau97}, which
lays emphasis on the simulations of the model for the emergence of
high level macro-states is fundamental to our framework, where the
observation process is by default based upon the simulations of
the model and not on analytical derivations.

Another limitation of the framework in its current state is that
it cannot be used effectively to make predictions regarding the
possible observable evolutionary dynamics in a ALife model during
simulations. This limitation though carries forward from the
nature of Darwinian theory which is too generic in its
conceptualization and based upon random sources of change that
make it difficult to derive useful predictions.

Similarly analysis of G$\ddot{o}$delian type conjunctures to
counter possibility of strong Alife, stating the impossibility of
formalizing life in general because that would imply formalizing
``mathematically intelligent" entities like ourselves, which could
in tern prove the G$\ddot{o}$del theorems in their own
``mathematical universe" having correspondence with ours, is also
beyond the scope of the current limits of the framework. See
\cite{sullins97,Rasmussen92}.
\subsubsection*{Problem of False Positives}
\label{false+ve}

Terms `false positive' and `false negative' are used in general to
highlight the limitations of `observation - inference' based
methodologies. \emph{False positive} refers to a situation where
observations and consequent inferences on a model result into a
claim of the presence of certain property in the model which
actually does not exist, while \emph{false negative} is used to
refer the situation where observations do not yield required
support for the presence of certain property, which is actually
present in the model. False negatives are usually the result of
incomplete observations while false positives indicate
arbitrariness in the observation/inference process.

Like any other generic specification framework, current framework
also suffers from the weakness of administering false positives.
False negatives are also possible, whereby an observation process
is defined such that it does not infer evolution, even though
there might actually be evolution present in the model. The case
of false negatives, however will not concern us since our focus is
to establish the presence of evolution in a given ALife model and
not whether it is absent with respect to certain observations. The
problem of false positives stems due to the fact that the
framework permits arbitrariness in the definition of entities and
their causal relationships. In case of causal relationships, they
are defined in the framework as observation dependent and might
not be consistent with the underlying micro-level dynamics of the
model (Section~\ref{reproformalized}). This arbitrariness might
give rise to false claims on the presence of evolution in the
model though there might be none actually.

For example, an observer (say $ob$) might decide to ``ignore"
entities in some states in the beginning and then choose later on
to observe them in some other states so that to use them for
establishing (false) evolutionary relationships, which would not
have been possible had he not preferred to ignore them earlier.
This problem of selectively observing entities in various states
requires additional constraints in the framework. We may add the
following constraint by considering another observer $ob'$ with
same universe of observation as $ob$. Let us consider a particular
simulation of a model as a state sequence $T$. For a state
subsequence $S$ of $T$, 
let $E_{ob}^S$ and $E_{ob'}^S$ denote the set of entities observed
by $ob$ and $ob'$ respectively. Consider that $ob'$ observes some
entities $\mathcal{X} \subseteq E_{ob'}^S$, which were ignored by
$ob$, that is, $\mathcal{X}\ \not\subseteq E_{ob}^S$. Now consider
the case when $ob$ chooses to observe $\mathcal{X}$ in some later
subsequence $S'$ of $T$, $S \neq S'$, that is, $\mathcal{X}$
$\subseteq E_{ob}^{S'}$, and also $\mathcal{X}$ $\subseteq
E_{ob'}^{S'}$, where $E_{ob}^{S'}$, and $E_{ob'}^{S'}$ are the
sets of entities observed by $ob$ and $ob'$ in $S'$. Now if $ob$
establishes evolutionary relationships using entities in
$\mathcal{X}$, which cannot be established by $ob'$, then we say
that $ob$ has drawn \emph{illegitimate} conclusions.

\subsection{Further work}
\label{sec:further-work}

Framework can be further extended in several interesting
directions, including the following: We need to capture the
essence of \emph{strong emergence} by considering several
observation processes at different organizational levels of the
model. We can also study overlapping evolutionary processes -
examples from real life include co-evolution, and sexual selection
versus environmental selection. Framework ought to be extended so
that fruitful predictions for a given ALife model regarding the
nature of evolutionary dynamics can be made. We also need to
introduce more strict constraints to overcome the problem of false
positives by limiting as to what could be claimed as observed.
Further insights can be gained by applying the framework to novel
classes of ALife models to refine the framework further, which we
are currently involved with.

\bibliographystyle{alpha}
\bibliography{alife}



\end{document}